\documentclass{article}

\usepackage[preprint]{neurips_2026}

\usepackage[utf8]{inputenc} 
\usepackage[T1]{fontenc}    
\usepackage{hyperref}       
\usepackage{url}            
\usepackage{booktabs}       
\usepackage{amsfonts}       
\usepackage{nicefrac}       
\usepackage{microtype}      
\usepackage{xcolor}         
\usepackage{amsmath,amssymb}
\usepackage{algorithm}
\usepackage{algpseudocode}
\usepackage{graphicx}
\usepackage{ulem} 
\usepackage{placeins}
\usepackage{subcaption}
\usepackage{amsthm} 
\newtheorem{thm}{Theorem} 

\newtheorem{lem*}{Lemma}

\title{TBP-mHC: full expressivity for manifold-constrained hyper connections through transportation polytopes}

\author{%
  Anton Lyubinin,
  \texttt{lastname@protonmail.com} \\
}

\begin{document}

\maketitle

\begin{abstract}
Hyper-Connections (HC) improve residual networks by introducing learnable mixing across multiple residual streams, but unconstrained mixing leads to training instability. Manifold-Constrained Hyper-Connections (mHC) address this by enforcing approximate double stochasticity via Sinkhorn normalization, while mHC-lite ensures exact constraints through convex combinations of permutation matrices at the cost of factorial complexity. KromHC reduces this cost using Kronecker-product parameterizations, but restricts the mixing matrices to a structured submanifold of the Birkhoff polytope .

We propose Transportation Birkhoff Polytope (TBP) parameterizations and their Recursive variants (RTBP), which construct exactly doubly stochastic mixing matrices with $(n-1)^2$ degrees of freedom. Our approach avoids iterative normalization and combinatorial explosion while preserving full expressivity of the Birkhoff polytope. Empirical results on language model pre-training' demonstrate competitive performance with improved stability and scalability.

The code is publicly available at \url{https://github.com/alyubinin/tbp-mHC}
\end{abstract}

\section{Introduction}

Residual connections \citep{he2016deep} are a fundamental component
of deep neural networks, defined by the update

\[
x_{l+1}=x_{l}+F_{l}(x_{l},W_{l}),
\]
which preserves an identity mapping across layers and stabilizes gradient
propagation. Recent work \citep{HC} generalizes this paradigm
through Hyper-Connections (HC), which expand a single residual stream
into $n$ parallel streams and introduce learnable mixing matrices.
A single HC layer is given by

\[
X_{l+1}=H_{l}^{\mathrm{res}}X_{l}+(H_{l}^{\mathrm{post}})^{\top}F_{l}\!\left(H_{l}^{\mathrm{pre}}X_{l},W_{l}\right),
\]
where $X_{l}\in\mathbb{R}^{n\times C}$ represents the expanded residue
stream, $H_{l}^{\mathrm{res}}\in\mathbb{R}^{n\times n}$ is a mixing
matrix, and $H_{l}^{\mathrm{pre}},H_{l}^{\mathrm{post}}\in\mathbb{R}^{1\times n}$
aggregate features from $n$ streams into one and map the layer output
back onto the $n$ streams respectively. This formulation enables
richer feature propagation across residual streams and improves expressivity.

However, replacing the identity with a general linear operator removes
the structural guarantees that stabilize deep compositions. In particular,
repeated products
\[
\prod_{i=1}^{L}H_{l+i}^{\mathrm{res}},
\]
that arise when multiple layers are composed, may deviate from identity-like
behavior, leading to loss of feature mean preservation and potential
gradient instability.

To address this problem, \citep{mHC} recent work constrains $H_{l}^{\mathrm{res}}$
to lie in the Birkhoff polytope \citep{birkhoff1946},

\[
\mathcal{B}_{n}\;=\;\left\{ X\in\mathbb{R}_{\ge0}^{n\times n}\;\middle|\;X\mathbf{1}=\mathbf{1},\;X^{\top}\mathbf{1}=\mathbf{1}\right\} ,
\]
i.e., the set of doubly stochastic (DS) matrices. Such matrices preserve
feature averages and have spectral norm bounded by 1, ensuring stability
under composition.

The Manifold-Constrained Hyper-Connections (mHC) framework from \citep{mHC}
enforces this constraint via the Sinkhorn--Knopp (SK) algorithm \citep{sinkhorn1967concerning},
which iteratively normalizes rows and columns:

\[
H_{l}^{res}=\mathrm{SK}({\rm exp}({\rm mat}(\alpha_{l}^{res}x_{l}'W_{l}^{res}+b_{l}^{res}))).
\]
While this procedure converges to a doubly stochastic matrix under
mild conditions, practical implementations use a finite number of
iterations (e.g., $K\approx20$ ), yielding only an approximation.
Classical results show that convergence can be arbitrarily slow, and
finite-iteration outputs may significantly violate the DS constraints,
leading to accumulated errors across depth (see discussion in \citep{mhc-lite}).
Moreover, efficient deployment requires specialized fused kernels
and recomputation strategies, increasing engineering complexity.

To eliminate the approximation gap, mHC-lite \citep{mhc-lite}
leverages the Birkhoff-von Neumann theorem \citep{birkhoff1946,Neumann53},
which states that matrix from $\mathcal{B}_{n}$ can be written as
a convex combination of permutation matrices, and thus:

\[
H_{l}^{res}=\sum_{k=1}^{K}\alpha_{k}P_{k},\qquad\alpha=\mathrm{Softmax}(\theta),\quad P_{k}\in\mathcal{P}_{n}.
\]
This parameterization guarantees exact doubly stochasticity and enables
straightforward implementation. However, it suffers from factorial
parameter growth, as representing the full polytope may require up
to $n!$ permutation matrices, making it quickly impractical with
the growth of $n$.

KromHC \citep{KromHC} addresses this scalability issue by exploiting
the closure of the Birkhoff polytope under Kronecker products:

\[
H^{\mathrm{res}}=U^{(K)}\otimes\cdots\otimes U^{(1)},\qquad U^{(i)}\in\mathcal{B}_{n_{i}}.
\]
This construction guarantees exact doubly stochasticity while reducing
parameter complexity to $O(n^{2}C)$ . However, it restricts the mixing
matrices to structured submanifold of $\mathcal{B}_{n}$ , thus limiting
expressivity (for detailed spectral analysis see \citep{gomhc2026}).
While the results of \citep{KromHC} show that in the tests, conducted
in the paper, KromHC is competitive, one cannot fully rule out the
possibility of such limitation affecting performance in certain scenarios.
Another issue is related to the restriction to the structured subset
in general. When the optimal transport lies outside this subset, optimization
proceeds with a nonzero normal gradient component, which can manifest
as late-stage gradient amplification or instability. In our limited
trials we have observed two such cases of KromHC behavior (see appendix
\ref{subsec:Gradient-unstable-KromHC-runs}). Our implementation differs
from the \citep{KromHC} optimizer split, which may affect smoothness,
but the use of Muon optimizer \citep{liu2025muon} would not by
itself remove the constrained-manifold issue. Also in \citep{KromHC}
Muon is applied only to the main branch, whereas the hyper-connection
parameters are still optimized with AdamW \citep{loshchilov2018decoupled},
used in our work.

These approaches expose a fundamental trade-off between Exactness,
Expressivity, Memory efficiency, and Speed.

The starting point of this work is recalling that the Birkhoff polytope
is a special case of a transportation polytope,

\[
\mathcal{T}(r,c)=\left\{ X\in\mathbb{R}_{\ge0}^{n\times n}\;\middle|\;X\mathbf{1}=r,\;X^{\top}\mathbf{1}=c\right\} ,
\]
with $r=c=\mathbf{1}$ . Transportation polytopes are fundamental
objects in operations research (see, for example, \citep{taha2017operations,Brualdi_2006}).
We revisit the problem from a transportation perspective and propose
the Transportation Birkhoff Polytope (TBP) parameterization for the
whole $\mathcal{B}_{n}$. Our construction is based on the classical
North-West Corner greedy algorithm \citep[Ch. 5.3.1]{taha2017operations}
and gives
\begin{itemize}
\item Exact doubly stochasticity (no approximation),
\item Minimal parameterization with $(n-1)^{2}$ degrees of freedom,
\item Full expressivity, covering all of $\mathcal{B}_{n}$.
\end{itemize}
TBP is optimal on exactness, expressivity, and memory efficiency,
and so the trade-off it makes must be the speed. It originates in
the sequential structure of the algorithm, which restricts parallelism.
It is a common feature of many optimal transport algorithms, limiting
their use in machine learning. 

To address this, we introduce Recursive TBP (RTBP), which decomposes
the problem into block subproblems. For example, partitioning the
matrix into four blocks:

\[
H_{l}^{res}=
\begin{pmatrix}
X_{11} & X_{12}\\
X_{21} & X_{22}\\
\end{pmatrix},
\]

we recursively construct each block subject to consistency constraints

\[
\sum X_{11}+\sum X_{12}=r_{\text{top}},\quad\sum X_{11}+\sum X_{21}=c_{\text{left}}.
\]

This hierarchical decomposition enables partial parallelization, which
increases the speed of the algorithm, while preserving exact feasibility.

We conduct numerical experiments to validate the effectiveness of
TBP and RTBP. We consider several modifications that make training
more effective and compare the performance with mHC, mHC-lite, and
KromHC. Our work does not consider the later approach of \citep{gomhc2026}.
We note that, however, \citep{gomhc2026} does not cover the full
Birkhoff polytope and uses a family of subsets, depending on a hyperparameter
that, with its growth, produce a dense subset of $\mathcal{B}_{n}$
in the limit.

Our results show that (R)TBP is a strong alternative competitor with
qualitative advantages. We also discover certain properties of other
algorithms that are not covered in their original papers.

In summary, our contributions are:
\begin{itemize}
\item We introduce TBP, an exact and minimal parameterization of the Birkhoff
polytope based on transportation polytope decomposition.
\item We develop RTBP, a recursive formulation that improves computational
efficiency via hierarchical parallelization.
\item We provide theoretical analysis showing that TBP and RTBP achieves
full expressivity with optimal parameterization.
\item We demonstrate empirically that TBP and RTBP achieve competitive performance
with mHC, mHC-lite, and KromHC, while offering a principled trade-off
between exactness, efficiency, expressivity, and speed.
\end{itemize}

\section{Background.}

\begin{table}[t]
\centering
\caption{Comparison of hyper-connection parameterizations.}
\label{tab:mhc-comparison}
\resizebox{\linewidth}{!}{%
\begin{tabular}{lcccccc}
\hline
\textbf{Method} & \textbf{Formula} & \textbf{Exactness} & \textbf{Expressivity} & \textbf{Memory} & \textbf{Speed} & \textbf{Stability} \\
\hline
mHC 
& $H^{res}=\mathrm{SK}(A)$ 
& $\times$ 
& full (approx.) 
& moderate 
& medium 
& moderate \\

mHC-lite 
& $H^{res}=\sum_{k=1}^K \alpha_k\,P_k,$ 
& \checkmark 
& full (as $K\!\to\! n!$) 
& poor 
& fast 
& high \\

& $\alpha=\mathrm{softmax}(\theta)$& & & & & \\

KromHC 
& $H^{res}=\bigotimes_{i=1}^m U_i$ 
& \checkmark 
& restricted 
& good 
& fast 
& moderate \\

TBP
& $H^{res}=\Phi_{\mathrm{TBP}}(t),$ 
& \checkmark 
& full 
& optimal 
& slow 
& high \\

RTBP 
& $H^{res}=\Phi_{\mathrm{RTBP}}(t),$ 
& \checkmark 
& full 
& optimal 
& medium 
& high \\

& $t\in\mathbb{R}^{(n-1)^2}$& & & & & \\

\hline
\end{tabular}}
\end{table}

\subsection{Residual Connections and Hyper-Connections (HC)}

Residual connections \citep{he2016deep} form the backbone of
modern deep neural networks by enabling stable optimization across
depth. A standard residual block updates a representation $x_{l}\in\mathbb{R}^{C}$
as

\[
x_{l+1}=x_{l}+F_{l}(x_{l},W_{l}),
\]
where $F_{l}:\mathbb{R}^{C}\to\mathbb{R}^{C}$ is a learnable transformation,
typically composed of attention and feed-forward sublayers. The identity
mapping ensures that information can propagate across layers without
degradation, mitigating vanishing and exploding gradient issues.

Hyper-Connections (HC, \citep{HC}) generalize this paradigm by
introducing multiple interacting residual streams. Instead of a single
vector, the representation at layer $l$ is $X_{l}\in\mathbb{R}^{n\times C}$,
where $n$ denotes the number of residual streams and each row corresponds
to a feature representation of dimension $C$ . The HC layer introduces
learnable mixing operators that propagate information across these
streams.

Formally, an HC layer is defined as

\[
X_{l+1}=H_{l}^{\mathrm{res}}X_{l}+(H_{l}^{\mathrm{post}})^{\top}F_{l}\!\left(H_{l}^{\mathrm{pre}}X_{l}\right),
\]
where $H_{l}^{\mathrm{res}}\in\mathbb{R}^{n\times n}$ is the residual mixing matrix, $H_{l}^{\mathrm{pre}}\in\mathbb{R}^{1\times n}$ aggregates streams before applying the nonlinear transformation $F_{l}$ ,
$H_{l}^{\mathrm{post}}\in\mathbb{R}^{1\times n}$ redistributes the transformed features back to the streams.

The term $H_{l}^{\mathrm{res}}X_{l}$ generalizes the identity connection
by allowing linear mixing across streams, while the second term corresponds
to a shared nonlinear transformation applied to a pooled representation
and then broadcasted back.

To make this explicit, let $\tilde{X}_{l}=H_{l}^{\mathrm{pre}}X_{l}\in\mathbb{R}^{1\times C}$, then $F_{l}(\tilde{X}_{l})\in\mathbb{R}^{1\times C}$,
and the broadcast step yields $(H_{l}^{\mathrm{post}})^{\top}F_{l}(\tilde{X}_{l})\in\mathbb{R}^{n\times C}$. Thus, each output stream is a combination of (i) linearly mixed previous
streams and (ii) a shared nonlinear update.

In the original HC, if we use $x_{l}$ for the flattened version of
$X_{l}$, concatenated from the $n\times C$ matrix into $1\times nC$
row-vector, then

\begin{equation}
\begin{cases}
\tilde{x}_{l}=\mathrm{RMSNorm}(x_{l}),\\[6pt]
H_{l}^{\mathrm{pre}}=\alpha_{l}^{\mathrm{pre}}\cdot\tanh\!\big(\theta_{l}^{\mathrm{pre}}\tilde{x}_{l}^{\top}\big)+b_{l}^{\mathrm{pre}},\\[6pt]
H_{l}^{\mathrm{post}}=\alpha_{l}^{\mathrm{post}}\cdot\tanh\!\big(\theta_{l}^{\mathrm{post}}\tilde{x}_{l}^{\top}\big)+b_{l}^{\mathrm{post}},\\[6pt]
H_{l}^{\mathrm{res}}=\alpha_{l}^{\mathrm{res}}\cdot\tanh\!\big(\theta_{l}^{\mathrm{res}}\tilde{x}_{l}^{\top}\big)+b_{l}^{\mathrm{res}},
\end{cases}\label{eq:hc}
\end{equation}
where $\mathrm{RMSNorm}(\cdot)$ refers to RMSNorm of \citep{zhang2019root}.

While HC increases expressivity by enabling cross-stream interactions,
it also introduces potential instability, as discussed in the introduction.

\subsection{Manifold-Constrained Hyper-Connections (mHC)}

Manifold-Constrained Hyper-Connections \citep{mHC} aims to correct
this problem by modifying the computation of $H_{l}^{pre}$ , $H_{l}^{post}$
and $H_{l}^{res}$, most importantly by constraining $H_{l}^{res}$
to the Birkhoff polytope $\mathcal{B}_{n}$ \citep{birkhoff1946},
i.e., the set of doubly stochastic matrices, 

\[
\mathcal{B}_{n}\;=\;\left\{ X\in\mathbb{R}_{\ge0}^{n\times n}\;\middle|\;X\mathbf{1}=\mathbf{1},\;X^{\top}\mathbf{1}=\mathbf{1}\right\} ,
\]
where $1_{n}$ denotes the vector of all ones and $X\ge0$ is elementwise.
Since doubly stochastic matrices have spectral norms equal to 1, and
the set is closed under matrix multiplication \citep{birkhoff1946},
this adjustment restores the stability of the residual mapping across
layers.

The constraint of $H_{l}^{res}$ to $\mathcal{B}_{n}$ is done using
Sinkhorn-Knopp projection \citep{sinkhorn1967concerning}, thus
giving the mappings for mHC in the following form:

\begin{equation}
\begin{cases}
x'_{l}=\mathrm{RMSNorm}(x_{l}),\\[6pt]
H_{l}^{\mathrm{pre}}=\sigma\!\big(\alpha_{l}^{\mathrm{pre}}\,x'_{l}W_{l}^{\mathrm{pre}}+b_{l}^{\mathrm{pre}}\big),\\[6pt]
H_{l}^{\mathrm{post}}=2\,\sigma\!\big(\alpha_{l}^{\mathrm{post}}\,x'_{l}W_{l}^{\mathrm{post}}+b_{l}^{\mathrm{post}}\big),\\[6pt]
H_{l}^{\mathrm{res}}=\mathrm{SK}\!\Big(\alpha_{l}^{\mathrm{res}}\cdot\mathrm{mat}\!\big(x'_{l}W_{l}^{\mathrm{res}}\big)+b_{l}^{\mathrm{res}}\Big),
\end{cases}\label{eq:mhc}
\end{equation}
where $W_{l}^{\mathrm{pre}},W_{l}^{\mathrm{post}}\in\mathbb{R}^{nC\times n},W_{l}^{\mathrm{res}}\in\mathbb{R}^{nC\times n^{2}}$
are learnable weight matrices, $b_{l}^{\mathrm{pre}},b_{l}^{\mathrm{post}}\in\mathbb{R}^{1\times n},b_{l}^{\mathrm{res}}\in\mathbb{R}^{1\times n^{2}}$
are learnable bias terms, $\alpha_{l}^{\mathrm{pre}},\alpha_{l}^{\mathrm{post}},\alpha_{l}^{\mathrm{res}}$
are learnable scalars, $\mathrm{mat}(\cdot):\mathbb{R}^{1\times n^{2}}\to\mathbb{R}^{n\times n}$
reshapes a vector into a matrix, $\sigma(\cdot)$ denotes the sigmoid
function, and $\mathrm{SK}(\cdot)$ is the Sinkhorn-Knopp operator. 

However Sinkhorn-Knopp algorithm is approximate and does not guarantee
double stochasticity in finite number of steps. This may lead to potential
instability, as shown on example in \citep{mhc-lite} with 20
number of iterations in SK algorithm. Also accelerating SK algorithm
requires customized kernels, thus increasing engineering complexity. 

\subsection{mHC-lite}

mHC-lite \citep{mhc-lite} presents a way to correct the limitations
of mHC, based on Birkhoff-von Neumann theorem \citep{birkhoff1946,Neumann53},
which states that any doubly stochastic matrix is a convex combination
of permutation matricies (see Table \ref{tab:mhc-comparison} and Appendix \ref{sec:Parametrizations}).

In this setting
$H_{l}^{\mathrm{res}}\in\mathbb{R}^{n\times n}$ is doubly stochastic
by construction. However, since there are $n!$ permutation matrices
of size $n\times n$, the parameter complexity grows as $O(nC\cdot n!)$,
preventing effective scaling.

\subsection{KromHC}

Motivated by the need of an alternative that is free from the limitations
of mHC and mHC-lite, \citep{KromHC} proposed KromHC parametrization.
\citep{KromHC} compose $H_{l}^{res}$ as Kronecker product of
smaller doubly stochastic matrices, based on factorization $n=i_{1}\cdot i_{2}\cdot\dots\cdot i_{K}$  (see Table \ref{tab:mhc-comparison} and Appendix \ref{sec:Parametrizations}). 

The number of learnable parameters in KromHC scales with $\Sigma\,i_{k}!$,
and since $i_{k}$ is usually small, compared to $n$, the parameter
complexity of KromHC grows as $O(Cn^{2})$.

The limitation of KromHC is that $H_{l}^{res}$ is restricted to a
structured subset of the Birkhoff polytop.

\citep{KromHC} conducted extensive experiments showing that this
does not prevent KromHC from achieving top results on various benchmarks,
making it a strong alternative to mHC and mHC-lite, despite drawbacks
of optimizing over a subset.

\subsection{Transportation polytopes \citep{taha2017operations,Brualdi_2006}}

Let $M=m(X)={\displaystyle \sum_{i,j}^{n,m}x_{ij}}$ be the total
mass of the matrix $X$. For given vectors $r=(r_{1},\dots,r_{n})$
and $c=(c_{1},\dots,c_{m})$ of row and column sums respectively, 

\[
r_{i}>0,\quad c_{j}>0,\quad\sum_{i}r_{i}=\sum_{j}c_{j}=M,
\]
the transportation polytope $\mathcal{T}(r,c)$ is the set of $n\times m$
matrices with given row and column sums, 

\[
\mathcal{T}(r,c)=\{X\in\mathbb{R}_{\ge0}^{n\times m}:\;X\mathbf{1}=r,\;X^{\top}\mathbf{1}=c\}.
\]
The topological dimension of $\mathcal{T}(r,c)$ is $(n-1)(m-1)$.

\subsubsection{Block decomposition of a transportation polytope}

The following construction is elementary, but is very hard to find in the literature.

Let $1\le k<n$ , $1\le l<m$ , and partitioning the indices $I_{1}=\{1,\dots,k\},$, $I_{2}=\{k+1,\dots,n\}$, $J_{1}=\{1,\dots,l\}$, and $J_{2}=\{l+1,\dots,m\}$.
Let also $r'$, $r''$, $c'$, and $c''$ be another set of positive margins, such that $r'+r''=r$, $c'+c''=c$, and partitioned vectors $r_{I_{1}}'=(r'_{i})_{i\in I_{1}}$, $r_{I_{1}}''=(r''_{i})_{i\in I_{1}}$, $r_{I_{2}}'=(r'_{i})_{i\in I_{2}}$, $r_{I_{2}}''=(r''_{i})_{i\in I_{2}}$, $c_{J_{1}}'=(c'_{j})_{j\in J_{1}}$, $c_{J_{2}}'=(c'_{j})_{j\in J_{2}}$, $c_{J_{1}}''=(c''_{j})_{j\in J_{1}}$, $c_{J_{2}}''=(c''_{j})_{j\in J_{2}}$ satisfy block balance conditions (See details in Appendix \ref{sec:Decomposition}).

Then we have a partitioning of $\mathcal{T}(r,c)$, that we write symbolically as

\begin{equation}
\mathcal{T}(r,c)=\bigcup_{r',r'',c',c''}\begin{pmatrix}\mathcal{T}(r'_{I_{1}},c'_{J_{1}}) & \mathcal{T}(r''_{I_{1}},c'_{J_{2}})\\
\mathcal{T}(r'_{I_{2}},c''_{J_{1}}) & \mathcal{T}(r''_{I_{2}},c''_{J_{2}})
\end{pmatrix}\label{eq:dectp}
\end{equation}

for all possible choices of $r'$, $r''$, $c'$, and $c''$, satisfying above conditions, i.e. a matrix $X$ can be written in block form

\[
X=\begin{pmatrix}X_{11} & X_{12}\\
X_{21} & X_{22}
\end{pmatrix},
\]
such that $X_{11}\in\mathcal{T}(r'_{I_{1}},c'_{J_{1}})$, $X_{12}\in\mathcal{T}(r''_{I_{1}},c'_{J_{2}})$, $X_{21}\in\mathcal{T}(r'_{I_{2}},c''_{J_{1}})$, $X_{22}\in\mathcal{T}(r''_{I_{2}},c''_{J_{2}})$.
For details see Appendix \ref{sec:Decomposition}.

\begin{table}[t]
\centering
\caption{Validation performance across experiments. Lower is better. Experiment 1 is the average of 5 runs with 5 different seeds. Experiments 2,3,4 are single run with the same seed.}
\label{tab:main}
\resizebox{\linewidth}{!}{%
\begin{tabular}{lcccc}
\hline
\textbf{Model} 
& \textbf{Exp. 1} 
& \textbf{Exp. 2} 
& \textbf{Exp. 3} 
& \textbf{Exp. 4} \\
& (loss / bpb) & (loss / bpb) & (loss / bpb) & (loss / bpb) \\
\hline
mHC        
& \textbf{3.48348} / \textbf{1.13203} 
& \uuline{3.49746} / \uuline{1.13848} 
& \uuline{3.27135} / \uuline{1.06487} 
& \uline{3.26951} / \uline{1.06427} \\

mHC-lite   
& 3.48742 / 1.13334 
& 3.50663 / 1.14146 
& \textbf{3.26555} / \textbf{1.06299} 
& \textbf{3.26105} / \textbf{1.06152} \\

KromHC     
& \uuline{3.48722} / \uuline{1.13330} 
& \textbf{3.49513} / \textbf{1.13772} 
& 3.27594 / 1.06637
& 3.29645 / 1.07304 \\

oRTBP-mHC  
& \uline{3.48589} / \uline{1.13282} 
& \uline{3.49743} / \uline{1.13846} 
& -- 
& \uuline{3.27600} / \uuline{1.06639} \\

LTBP-mHC   
& -- 
& 3.49952 / 1.13914 
& -- 
& -- \\

lmaLTBP-mHC 
& -- 
& 3.50218 / 1.14001 
& 3.27336 / 1.06553 
& -- \\

amsTBP-mHC 
& -- 
& 3.50336 / 1.14040 
& -- 
& -- \\

msTBP-mHC  
& -- 
& 3.50599 / 1.14112 
& \uline{3.26911} / \uline{1.06414} 
& -- \\

asTBP-mHC  
& -- 
& 3.50599 / 1.14125 
& 3.28178 / 1.06827 
& -- \\

RTBP-mHC   
& -- 
& -- 
& 3.27767 / 1.06667 
& -- \\

sRTBP-mHC  
& -- 
& -- 
& 3.27778 / 1.06697 
& -- \\

aTBP-mHC   
& -- 
& -- 
& 3.27438 / 1.06586 
& -- \\

TBP-mHC    
& -- 
& -- 
& 3.27767 / 1.06693 
& -- \\

\hline
\end{tabular}
}
\end{table}

\section{Method}

TBP changes parametrization of $H_{l}^{res}$ to a global chart for
the interior of the Birkhoff polytope. This guarantees the exact double
stochasticity, full expressivity and minimal parameter count.

The full algorithms are given in appendix \ref{sec:Algorithms}, proofs
of the theorems are in appendix \ref{sec:Proofs}.

\subsection{Transportation Birkhoff Polytope (TBP) Algorithm}

TBP generates a doubly stochastic matrix $X$ by directly solving
transportation polytope and generating elements of $X$ row-by-row,
from left to right, from free parameters $(t_{ij})_{i,j=1}^{n-1}$
updating the remaining budgets in rows and columns on every step. 

Let on the step $(i,j)$ the $(r_{i})_{i=1}^{n}$ and $(c_{j})_{j=1}^{m}$
be the remaining budgets in rows and columns respectively. Then the
feasibility region for $x_{ij}$ is the interval $[L_{ij},U_{ij}]$,
with 

- $L_{ij}=\max\left(0,\;r_{i}-\sum_{k=j+1}^{n}c_{k},\;c_{j}-\sum_{k=i+1}^{n}r_{k}\right)$
- forces $x_{ij}$ to be large enough so that the remaining sum of
the row $r_{i}$ doesn't exceed the total remaining capacity of the
remaining columns;

- $U_{ij}=\min(r_{i},\;c_{j})$ - you cannot ship more than a row
supplies or a column demands.

Then we generate $x_{ij}$:

\[
x_{ij}=L_{ij}+\Delta_{ij}\,\sigma(t_{ij})
\]

where $\sigma(t)$ is the sigmoid function, $\Delta_{ij}=U_{ij}-L_{ij}$,
and update the budgets:

\[
r_{i}\leftarrow r_{i}-x_{ij},\quad c_{j}\leftarrow c_{j}-x_{ij}.
\]

The last element in every row is $x_{in}=r_{i}$, and for the last
row $x_{nj}=c_{j}$.
\begin{thm}
\label{thm:TBP-defines-a}TBP defines a bijective chart on the interior
of $\mathcal{T}(r,c)$.
\end{thm}

Applied to generating $H_{l}^{res}$ in the Birkhoff polytope, the
structure of TBP-HC is the following:

\[
\begin{cases}
\hat{x}'_{l}=\mathrm{RMSNorm}(\hat{x}_{l}),\\[6pt]
H_{l}^{\mathrm{pre}}=\sigma\!\big(\alpha_{l}^{\mathrm{pre}}\,\hat{x}'_{l}W_{l}^{\mathrm{pre}}+b_{l}^{\mathrm{pre}}\big),\\[6pt]
H_{l}^{\mathrm{post}}=2\,\sigma\!\big(\alpha_{l}^{\mathrm{post}}\,\hat{x}'_{l}W_{l}^{\mathrm{post}}+b_{l}^{\mathrm{post}}\big),\\[6pt]
H_{l}^{\mathrm{res}}={\displaystyle {\mathrm{TBP}(\hat{x}'_{l}W_{l}^{\mathrm{res}}+b_{l}^{\mathrm{res}})},}
\end{cases}
\]

where all parameters have the same meaning as in mHC.

The parameter complexity of TBP-HC is similar to mHC,

\[
P_{\mathrm{TBP,\ current}}\approx n^{3}C+(2C+1)n^{2}+2n+3.
\]

\subsection{Recursive Transportation Birkhoff Polytope (RTBP) Algorithm}

The main drawback of TBP is its speed. As a fully sequential algorithm
- it looses the speed competition to other mHC variants. We attempt
to improve its performance with Recursive Transportation Birkhoff
Polytope (RTBP) algorithm.

RTBP constructs $X\in\mathcal{T}(r,c)$ by recursively splitting the
matrix into $2\times2$ blocks, ensuring feasibility, using the decomposition
from (\ref{eq:dectp}), and then solving a smaller transportation
polytope.

Let $I_{1},I_{2},J_{1},J_{2}$ be as in (\ref{eq:part}), 
$R_{1}=\sum_{i\in I_{1}}r_{i}$, $R_{2}=\sum_{i\in I_{2}}r_{i}$, $C_{1}=\sum_{j\in J_{1}}c_{j}$, $C_{2}=\sum_{j\in J_{2}}c_{j}$.

We are going to construct

\[
X=\begin{pmatrix}X_{11} & X_{12}\\
X_{21} & X_{22}
\end{pmatrix},
\]
with total masses $M=m(X)$, $M_{11}=m(X_{11})$, $M_{12}=m(X_{12})$,
$M_{21}=m(X_{21})$, $M_{22}=m(X_{22})$, from $(n-1)(m-1)$ free
parameters $t_{ij}$.

Feasibility requires $\max\big(0,\;R_{1}-C_{2}\big)\;\le\;M_{11}\;\le\;\min(R_{1},\;C_{1})$,
and then 

\[
M_{12}=R_{1}-M_{11},\quad M_{21}=C_{1}-M_{11},\quad M_{22}=R_{2}-M_{21}=C_{2}-M_{12}.
\]
Thus the choice of $M_{11}$ defines the splitting of the mass.

The partitioned margins $r'$, $r''$, $c'$, and $c''$, such that
$r'+r''=r$, $c'+c''=c$, with partitioned vectors $r_{I_{1}}'=(r'_{i})_{i\in I_{1}}$, $r_{I_{1}}''=(r''_{i})_{i\in I_{1}}$, $r_{I_{2}}'=(r'_{i})_{i\in I_{2}}$, $r_{I_{2}}''=(r''_{i})_{i\in I_{2}}$, $c_{J_{1}}'=(c'_{j})_{j\in J_{1}}$, $c_{J_{2}}'=(c'_{j})_{j\in J_{2}}$, $c_{J_{1}}''=(c''_{j})_{j\in J_{1}}$, $c_{J_{2}}''=(c''_{j})_{j\in J_{2}}$ 
have to satisfy block balance conditions:

\[
\sum_{i\in I_{1}}r'_{i}=\sum_{j\in J_{1}}c'_{j}=M_{11},\qquad\sum_{i\in I_{1}}r''_{i}=\sum_{j\in J_{2}}c'_{j}=M_{12},
\]

\[
\sum_{i\in I_{2}}r'_{i}=\sum_{j\in J_{1}}c''_{j}=M_{21},\qquad\sum_{i\in I_{2}}r''_{i}=\sum_{j\in J_{2}}c''_{j}=M_{22}.
\]

Thus we get $X_{11}\in\mathcal{T}(r'_{I_{1}},c'_{J_{1}})$, $X_{12}\in\mathcal{T}(r''_{I_{1}},c'_{J_{2}})$,
$X_{21}\in\mathcal{T}(r'_{I_{2}},c''_{J_{1}})$, $X_{22}\in\mathcal{T}(r'_{I_{2}},c''_{J_{2}}),$
with each of them solved recursively. When recursion reaches $n$
or $m$ equal to 1 - the respective polytope dimension is 0, with
only one immediate solution (in practice, we construct the polytopes
explicitly already for $2\times2$ case).
\begin{thm}
\label{thm:RTBP-algorithm-uses}RTBP algorithm uses exactly $(n-1)(m-1)$
free parameters.
\end{thm}

Thus the structure of RTBP-HC is similar to TBP-HC:

\[
\begin{cases}
\hat{x}'_{l}=\mathrm{RMSNorm}(\hat{x}_{l}),\\[6pt]
H_{l}^{\mathrm{pre}}=\sigma\!\big(\alpha_{l}^{\mathrm{pre}}\,\hat{x}'_{l}W_{l}^{\mathrm{pre}}+b_{l}^{\mathrm{pre}}\big),\\[6pt]
H_{l}^{\mathrm{post}}=2\,\sigma\!\big(\alpha_{l}^{\mathrm{post}}\,\hat{x}'_{l}W_{l}^{\mathrm{post}}+b_{l}^{\mathrm{post}}\big),\\[6pt]
H_{l}^{\mathrm{res}}={\displaystyle {\mathrm{RTBP}(\hat{x}'_{l}W_{l}^{\mathrm{res}}+b_{l}^{\mathrm{res}})},}
\end{cases}
\]
where all parameters have the same meaning as in mHC, and the same
parameter complexity.

\section{Experiments}

We implement all methods extending the framework of \citep{mhc-lite}
and \citep{KromHC}, which in their turn extend nanoGPT training
framework \citep{karpathy2022nanogpt}. We have refactored the
implementation, which has improved the performance of all models.
Due to computational constraint, we run four different experiments
on four different nodes and pre-train all models on OpenWebText \citep{Gokaslan2019OpenWeb}
only, with 10000 training iterations. We consider two main model scales.
The Small configuration uses 6 Transformer layers, 8 attention heads,
and embedding dimension 512, with peak learning rate $1e-3$. The
Medium configuration uses 12 layers, 12 heads, and embedding dimension
768, with peak learning rate $6e-4$. Throughout all experiments the
number of residual streams is set to 4. 

The runs of KromHC with gradient instability were not included into
these results.

Further details, including initialization settings, can be found in appendix \ref{sec:Experiments.}.

\subsection{Implementation.}

We evaluate several implementation variants of the same exact TBP chart. 
Scaled TBP (sTBP) rescales local logits to reduce vanishing sensitivity in narrow intervals; 
margined TBP (mTBP) keeps choices away from interval boundaries to prevent saturation; averaged TBP (aTBP) reduces ordering bias by averaging over permutations; lazy/minorized variants (lmTBP) enforce additional spectral mixing; linearized TBP (LTBP) uses raw logits without $\sigma(\cdot)$ function (clipped between 0 and 1).

Same modifications apply to RTBP. Optimized RTBP (oRTBP) places recursive chart parameters to a separate optimizer group with no weight decay.

Variants are usually combined as msTBP - margined+scaled, amsTBP - adds averaging, and so on.

Full definitions are given in Appendix \ref{sec:Implementations.}.

\subsection{Training and validation results}

We compare the performance of TBP and RTBP variants with mHC, mHC-lite
and KromHC. We report the cross-entropy loss $\mathcal{L}_{CE}=-\mathbb{E}\left[\log p_{\theta}\left(x_{t}\mid x<t\right)\right]$
and bits-per-byte $BPB=\frac{\mathcal{L}_{CE}}{\ln2\cdot\text{tokens\_per\_byte}}$
for both training and validation. We found that the results are volatile
due to changes of the seed and hardware configuration. In table \ref{tab:main}
we report validation loss and BPB at the end of training, training results can be found in appendix  \ref{sec:Experiments.}.

Overall, in each experiment the models from TBP and RTBP families
showed competitive results. 

\subsection{Gradient norms}

Figure \ref{fig:gradient-norms} represents the gradient norms in all four experiments.
The results show that (R)TBP variants consistently achieve lower gradient
norms, compared to mHC, mHC-lite, and KromHC, confirming its greater
stability in training.

\begin{figure}[t]
\centering

\begin{subfigure}{0.48\linewidth}
    \centering
    \includegraphics[width=\linewidth]{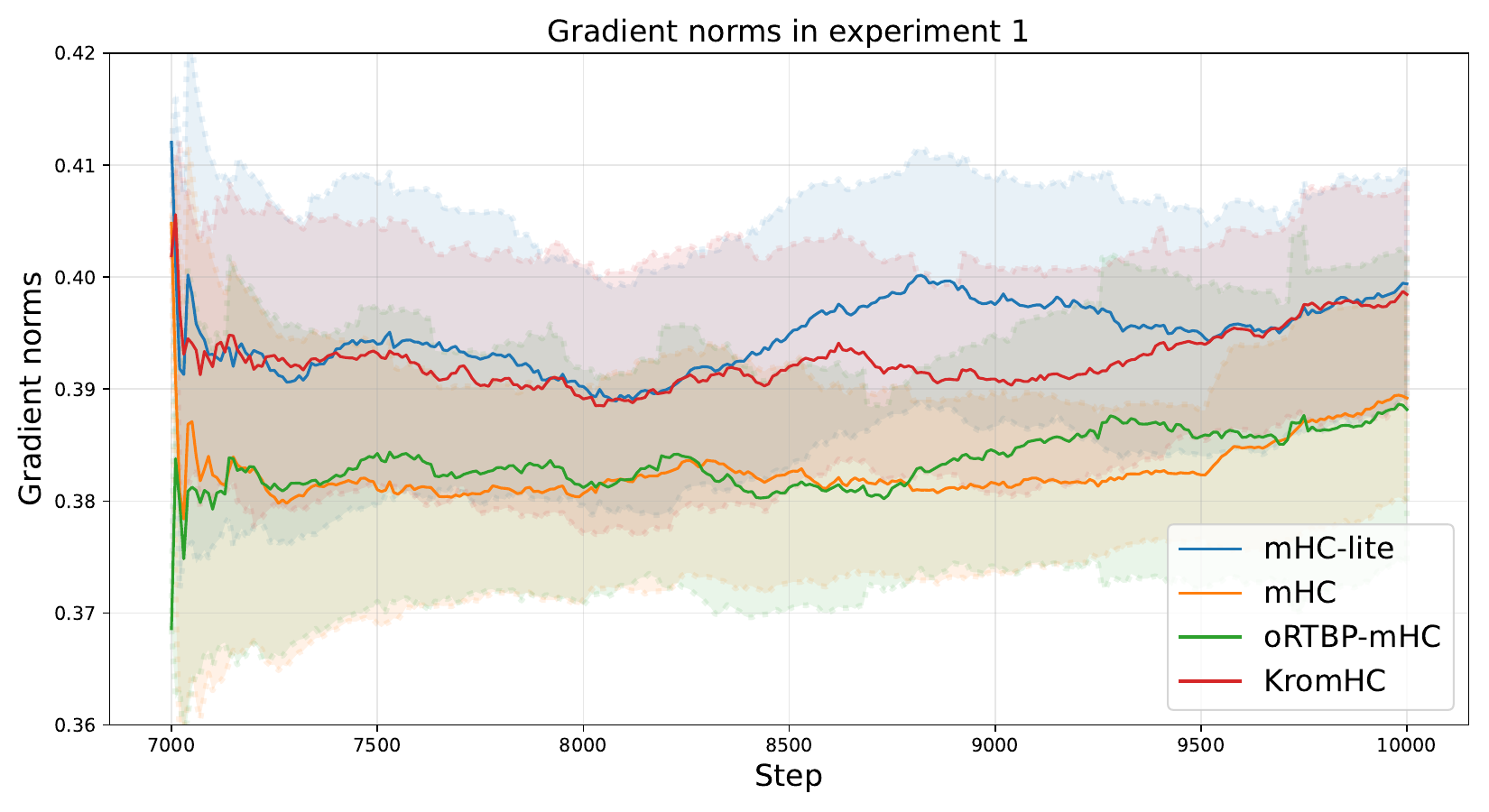}
    \caption{Experiment 1}
    \label{fig:grad-norm-1}
\end{subfigure}
\hfill
\begin{subfigure}{0.48\linewidth}
    \centering
    \includegraphics[width=\linewidth]{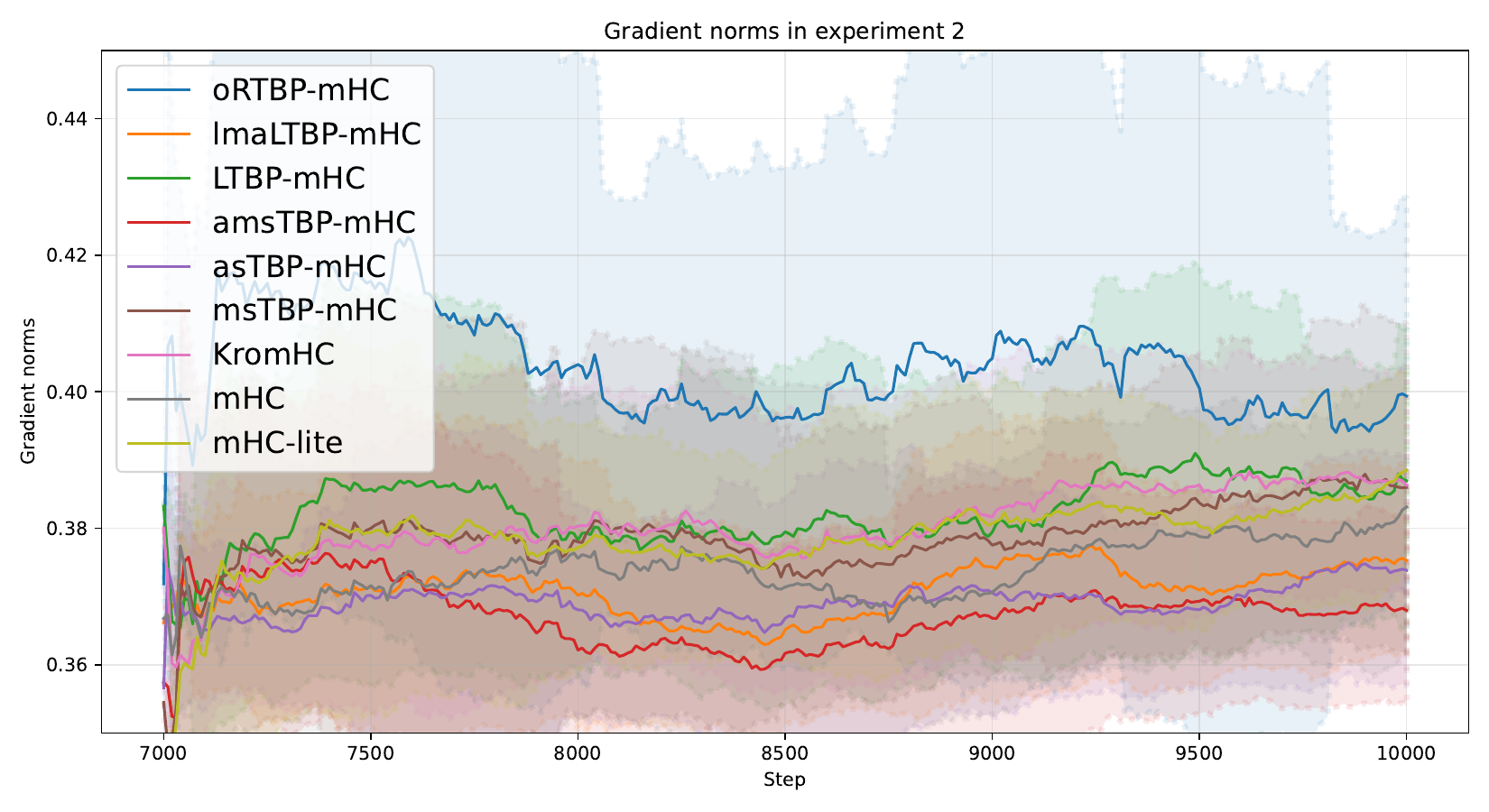}
    \caption{Experiment 2}
    \label{fig:grad-norm-2}
\end{subfigure}

\vspace{0.4em}

\begin{subfigure}{0.48\linewidth}
    \centering
    \includegraphics[width=\linewidth]{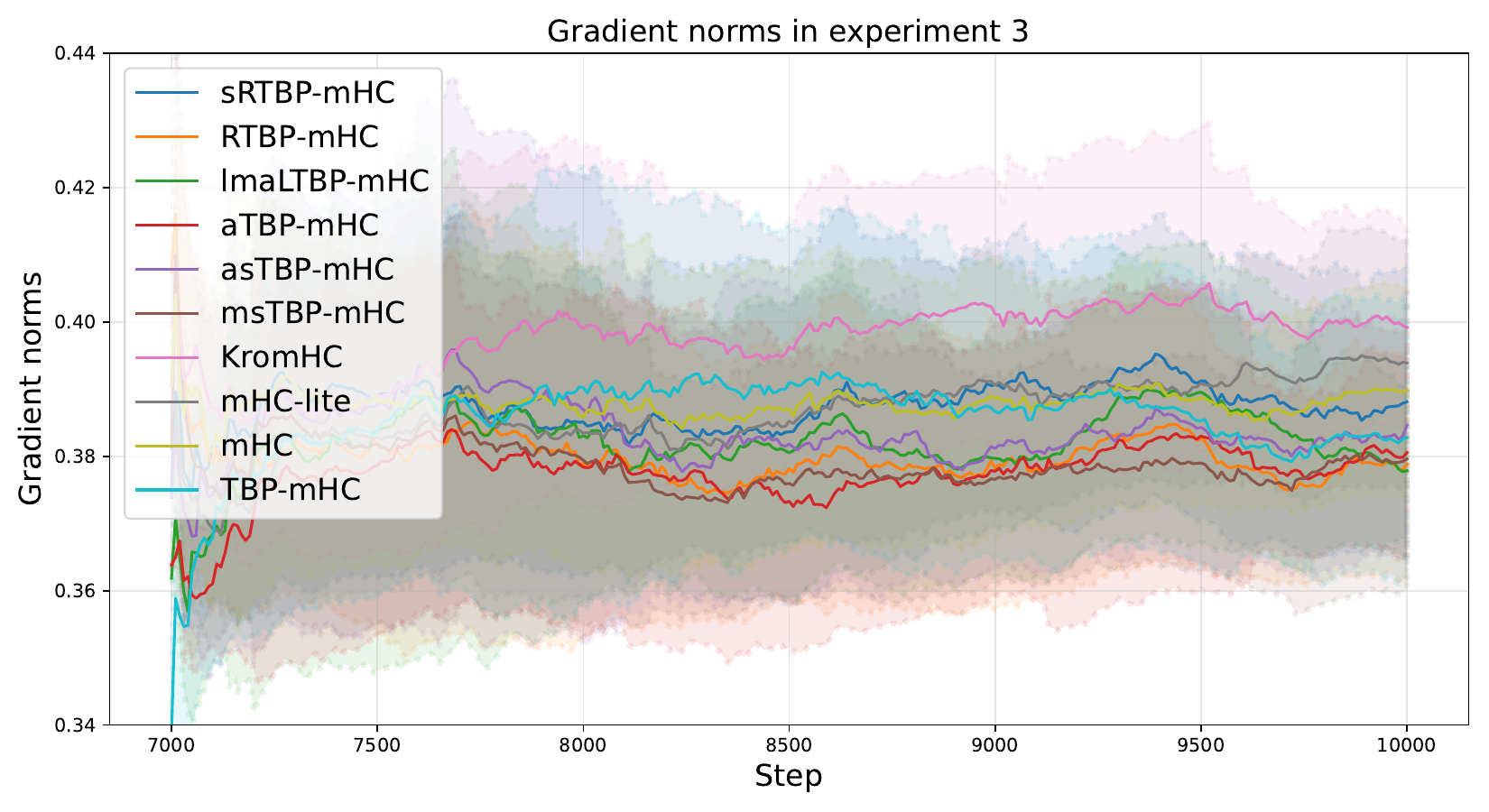}
    \caption{Experiment 3}
    \label{fig:grad-norm-3}
\end{subfigure}
\hfill
\begin{subfigure}{0.48\linewidth}
    \centering
    \includegraphics[width=\linewidth]{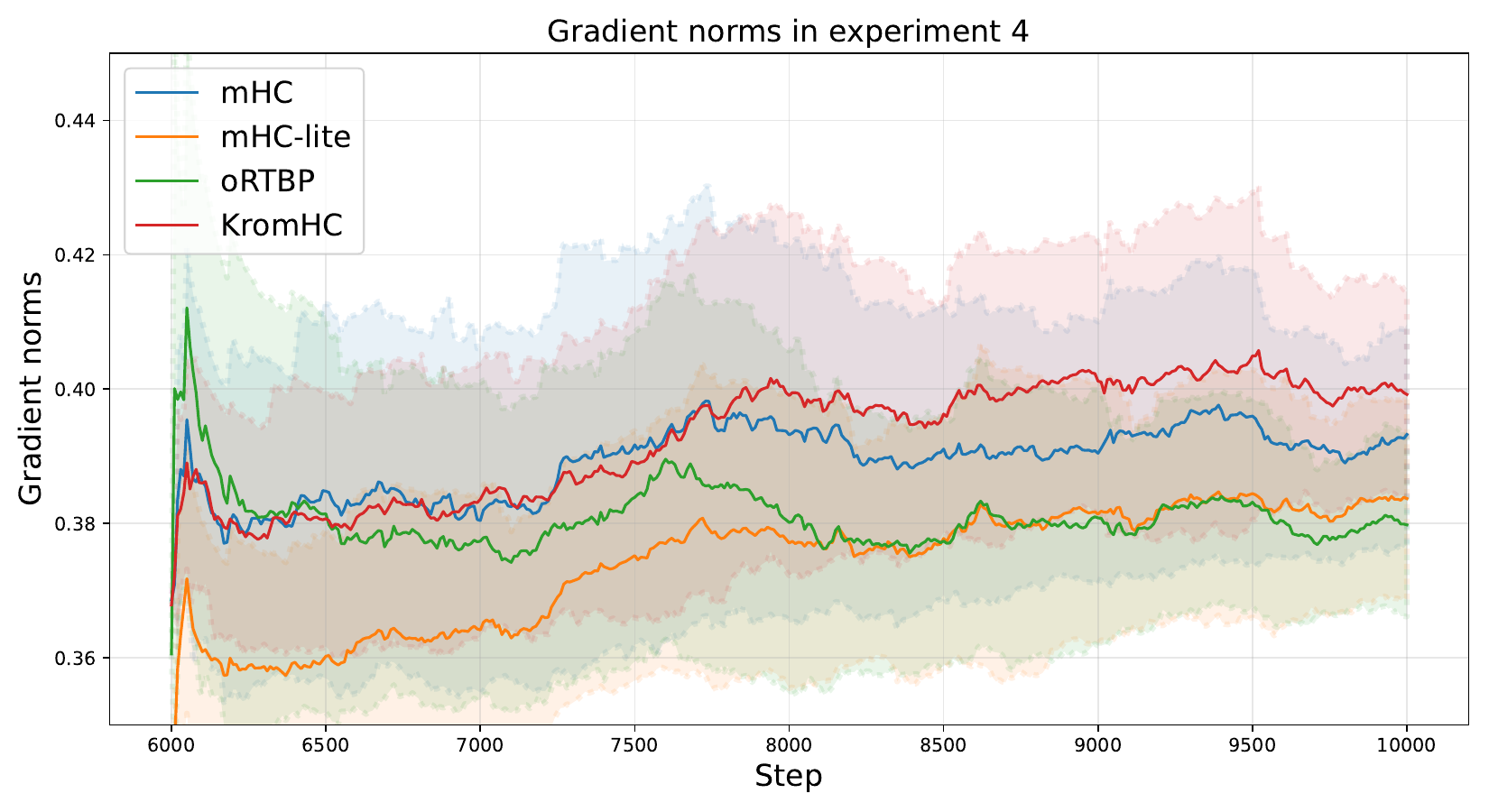}
    \caption{Experiment 4}
    \label{fig:grad-norm-4}
\end{subfigure}

\caption{Gradient norm dynamics across four experiments. Each panel reports the gradient norm over training iterations for a different configuration.}
\label{fig:gradient-norms}
\end{figure}

\section{Conclusion}

We introduced TBP and RTBP, transport-based parameterizations for manifold-constrained hyper-connections that address the limitations of prior approaches: Sinkhorn approximation error, Birkhoff--von Neumann factorial complexity, and Kronecker structural rigidity. These methods construct residual matrices exactly in the Birkhoff polytope using only $(n-1)^2$ degrees of freedom, preserving full expressivity without iterative projection or combinatorial explosion. RTBP improves symmetry and parallelism, and optimizer-aware variants such as oRTBP further stabilize training. Experiments on language-model pretraining show that exact doubly stochastic constraints can be achieved without sacrificing scalability or performance.

\subsection*{Limitations and future work.}

Despite very competitive results, current implementations of (R)TBP
do not achieve their full theoretical potential. This is due to the
fact that (R)TBP-based methods introduce nontrivial couplings and
nonlinear structure (especially recursive methods), which may complicate
optimization and implementation, particularly for very large n. 
Future work will explore implementations for improved training dynamics, in particular
under modern improved optimizers, such as AdamN \citep{electronics15030670}, Muon \citep{liu2025muonscalablellmtraining}, and Sophia \citep{liu2024sophiascalablestochasticsecondorder}.




\appendix

\section{Block decomposition of a transportation polytope}\label{sec:Decomposition}

Decomposing transportation polytopes goes back to Danzig \citep{dantzig1998linear}
for counting the number of solutions. The following construction is
elementary, but is very hard to find in the literature.

Fix integers $1\le k<n$ , $1\le l<m$ , and partition the indices:

\begin{equation}
I_{1}=\{1,\dots,k\},\quad I_{2}=\{k+1,\dots,n\},\qquad J_{1}=\{1,\dots,l\},\quad J_{2}=\{l+1,\dots,m\}.\label{eq:part}
\end{equation}
Let $r'$, $r''$, $c'$, and $c''$ be another set of positive margins
such that $r'+r''=r$, $c'+c''=c$, and partitioned vectors 

\[
r_{I_{1}}'=(r'_{i})_{i\in I_{1}},\quad r_{I_{1}}''=(r''_{i})_{i\in I_{1}},\quad r_{I_{2}}'=(r'_{i})_{i\in I_{2}},\quad r_{I_{2}}''=(r''_{i})_{i\in I_{2}},
\]

\[
c_{J_{1}}'=(c'_{j})_{j\in J_{1}},\quad c_{J_{2}}'=(c'_{j})_{j\in J_{2}},\quad c_{J_{1}}''=(c''_{j})_{j\in J_{1}},\quad c_{J_{2}}''=(c''_{j})_{j\in J_{2}}
\]
satisfy block balance conditions:

\[
\sum_{i\in I_{1}}r'_{i}=\sum_{j\in J_{1}}c'_{j},\qquad\sum_{i\in I_{1}}r''_{i}=\sum_{j\in J_{2}}c'_{j},
\]

\[
\sum_{i\in I_{2}}r'_{i}=\sum_{j\in J_{1}}c''_{j},\qquad\sum_{i\in I_{2}}r''_{i}=\sum_{j\in J_{2}}c''_{j}.
\]
Then a matrix 

\[
X=\begin{pmatrix}X_{11} & X_{12}\\
X_{21} & X_{22}
\end{pmatrix},
\]
such that

\[
X_{11}\in\mathcal{T}(r'_{I_{1}},c'_{J_{1}}),\quad X_{12}\in\mathcal{T}(r''_{I_{1}},c'_{J_{2}}),\quad X_{21}\in\mathcal{T}(r'_{I_{2}},c''_{J_{1}}),\quad X_{22}\in\mathcal{T}(r''_{I_{2}},c''_{J_{2}}),
\]
belongs to $X\in\mathcal{T}(r,c)$.

Conversely, every matrix $X\in\mathcal{T}(r,c)$ can be block-partitioned
into 

\[
X=\begin{pmatrix}X_{11} & X_{12}\\
X_{21} & X_{22}
\end{pmatrix},
\]
with

\[
X_{11}\in\mathcal{T}(r'_{I_{1}},c'_{J_{1}}),\quad X_{12}\in\mathcal{T}(r''_{I_{1}},c'_{J_{2}}),\quad X_{21}\in\mathcal{T}(r'_{I_{2}},c''_{J_{1}}),\quad X_{22}\in\mathcal{T}(r''_{I_{2}},c''_{J_{2}}).
\]
for the respective choice of $r'$, $r''$, $c'$, and $c''$.

Thus we can say that for a fixed partitioning of indices $I_{1},I_{2},J_{1},J_{2}$,
we have a partitioning of $\mathcal{T}(r,c)$, that we write symbolically
as

\begin{equation}
\mathcal{T}(r,c)=\bigcup_{r',r'',c',c''}\begin{pmatrix}\mathcal{T}(r'_{I_{1}},c'_{J_{1}}) & \mathcal{T}(r''_{I_{1}},c'_{J_{2}})\\
\mathcal{T}(r'_{I_{2}},c''_{J_{1}}) & \mathcal{T}(r''_{I_{2}},c''_{J_{2}})
\end{pmatrix}\label{eq:dectp}
\end{equation}

for all possible choices of $r'$, $r''$, $c'$, and $c''$, satisfying
above conditions.


\section{Parametrizations for mHC, mHC-lite, KromHC}\label{sec:Parametrizations}

In mHC-lite the structure of mHC remains unchanged, except for $H_{l}^{\mathrm{res}}$: 

\[
\begin{cases}
\hat{x}'_{l}=\mathrm{RMSNorm}(\hat{x}_{l}),\\[6pt]
H_{l}^{\mathrm{pre}}=\sigma\!\big(\alpha_{l}^{\mathrm{pre}}\,\hat{x}'_{l}W_{l}^{\mathrm{pre}}+b_{l}^{\mathrm{pre}}\big),\\[6pt]
H_{l}^{\mathrm{post}}=2\,\sigma\!\big(\alpha_{l}^{\mathrm{post}}\,\hat{x}'_{l}W_{l}^{\mathrm{post}}+b_{l}^{\mathrm{post}}\big),\\[6pt]
a_{l}=\mathrm{softmax}\!\big(\alpha_{l}^{\mathrm{res}}\,\hat{x}'_{l}W_{l}^{\mathrm{res}}+b_{l}^{\mathrm{res}}\big),\\[6pt]
H_{l}^{\mathrm{res}}={\displaystyle \sum_{k=1}^{n!}a_{l,k}\,P_{k},}
\end{cases}
\]
where all elements defined in (\ref{eq:mhc}) retain their meaning,
and $\{P_{k}\}_{k=1}^{n!}$ are permutation matrices. 

In KromHC, if $n=i_{1}\cdot i_{2}\cdot\dots\cdot i_{K}$
is factored into $K$ terms, 

\begin{equation}
\begin{cases}
H_{l}^{\mathrm{pre}}=\sigma\!\big(\alpha_{l}^{\mathrm{pre}}\,x'_{l}W_{l}^{\mathrm{pre}}+b_{l}^{\mathrm{pre}}\big),\\[6pt]
H_{l}^{\mathrm{post}}=2\,\sigma\!\big(\alpha_{l}^{\mathrm{post}}\,x'_{l}W_{l}^{\mathrm{post}}+b_{l}^{\mathrm{post}}\big),\\[6pt]
a_{l}^{(k)}=\mathrm{Softmax}\!\big(\alpha_{l}^{\mathrm{res}}\,x'_{l}W_{l}^{\mathrm{res},k}+b_{l}^{\mathrm{res},k}\big),\\[6pt]
U_{l}^{(k)}={\displaystyle \sum_{m=1}^{i_{k}!}a_{l}^{(k)}(m)\,P_{m},}\\[6pt]
H_{l}^{\mathrm{res}}=\bigotimes_{k=1}^{K}U_{l}^{(k)},
\end{cases}\label{eq:krom}
\end{equation}
where $P_{m}\in\mathbb{R}^{i_{k}\times i_{k}}$ is the $m$-th permutation
matrix, $a_{l}^{(k)}(m)$ is the $m$-th entry of $a_{l}^{(k)}$,
and $\bigotimes_{k=1}^{K}$ denotes the Kronecker product over factors
$U_{l}^{(k)}$. 

Each factor $U_{l}^{(k)}$ lies in the Birkhoff polytope $\mathcal{B}_{i_{k}}$
as a convex combination of permutation matrices, hence $H_{l}^{\mathrm{res}}$
is doubly stochastic by closure under Kronecker products. The factorization
$\{i_{k}\}_{k=1}^{K}$ can be any decomposition of $n$ with $i_{k}\ge2$;
choosing a prime factorization (e.g., $i_{k}=2$ for $n=2^{K}$) yields
the highest parameter efficiency.


\section{Proofs}\label{sec:Proofs}

\subsubsection*{Proof of the theorem \ref{thm:TBP-defines-a}.}

Every step of the TBP algorithm is reversible: for the given matrix
$X$, on the step $(i,j)$ on the element $x_{ij}$, let the $(r_{i})_{i=1}^{n}$
and $(c_{j})_{j=1}^{m}$ be the remaining budgets in rows and columns
respectively. Let also $L_{ij}$ and $U_{ij}$ be the TBP feasibility
boundaries. Then 
\[
t_{ij}=\sigma^{-1}\left(\frac{x_{ij}-L_{ij}}{\Delta_{ij}}\right),
\]
where $\sigma^{-1}$ is the inverse function to sigmoid. This defines
a map that is the inverse to the TBP chart on the step $(i,j)$, and
so the chart is invertible. Thus it is bijective.

\subsubsection*{Proof of the theorem \ref{thm:RTBP-algorithm-uses}.}

Let $D(n,m)$ denote the number of scalar parameters used by RTBP
on an $n\times m$ transportation polytope. Then 

\[
D(n,m)=(n-1)(m-1).
\]

\begin{lem*}
If an $n\times m$ transportation problem is split into a $2\times2$
block decomposition with row group sizes $(n_{1},n_{2})$ and column
group sizes $(m_{1},m_{2})$ , then this splitting requires $n+m-3$
parameters.
\end{lem*}
\begin{proof}
The choice of $r'$ and $c'$ determines the choice of $r''$ and
$c''$ respectively. With fixed $M_{11}$, block balance conditions
imply that the choice of $r_{I_{1}}'$ and $r_{I_{2}}'$ involves
$n_{1}-1$ and $n_{2}-1$ parameters respectively. Similarly the choice
of $c_{J_{1}}'$ and $c_{J_{2}}'$ involves $m_{1}-1$ and $m_{2}-1$
parameters respectively. Coupled with the choice of $M_{11}$, this
imply that the choice of partition for $X$ involves 
\[
(n_{1}-1)+(n_{2}-1)+(m_{1}-1)+(m_{2}-1)+1=n+m-3
\]
free parameters.
\end{proof}
We will now prove the theorem by induction.

If $n=1$ or $m=1$ , the transportation polytope is a single point,
so 

\[
D(1,m)=D(n,1)=0.
\]
This agrees with

\[
(1-1)(m-1)=0,\qquad(n-1)(1-1)=0.
\]

In the general case, suppose we split rows into two groups of sizes
$n=n_{1}+n_{2}$ and columns into two groups of sizes $m=m_{1}+m_{2}$.
Then RTBP produces four subproblems: $(n_{1},m_{1})$, $(n_{1},m_{2})$,
$(n_{2},m_{1})$, $(n_{2},m_{2})$, with the theorem holding for them
by inductive assumption. According to the lemma, the splitting contributes
$n+m-3$ parameters. Then in total we have

\[
\begin{array}{c}
D(n,m)=(n+m-3)+D(n_{1},m_{1})+D(n_{1},m_{2})+D(n_{2},m_{1})+D(n_{2},m_{2})\\
=(n+m-3)+(n_{1}-1)(m_{1}-1)+(n_{1}-1)(m_{2}-1)+(n_{2}-1)(m_{1}-1)+(n_{2}-1)(m_{2}-1)\\
=n+m-3+(n_{1}-1)\big((m_{1}-1)+(m_{2}-1)\big)+(n_{2}-1)\big((m_{1}-1)+(m_{2}-1)\big)\\
=n+m-3+(n_{1}-1)(m-2)+(n_{2}-1)(m-2)=n+m-3+(n-2)(m-2)\\
=n+m-3+nm-2n-2m+4=nm-n-m+1=(n-1)(m-1).
\end{array}
\]

\section{Lazyfication}\label{sec:Lazyfication}

Given a mixing matrix $H\in\mathbb{R}^{\text{n\texttimes n}}$, following
\citep{fried2021alphalazyversionmarkovchains}, \citep{10.1561/0400000003},
\citep{hermon2016maximal}, we define its $\alpha$-lazy version
by 

\[
H^{(\alpha)}=(1-\alpha)I_{n}+\alpha H,\alpha\in\left(0,1\right].
\]
This is a convex combination of the identity and H, preserving nonnegativity
and double stochasticity.

Since $I_{n}$ commutes with $H$, the eigenvalues transform as 

\[
\lambda_{i}(H^{(\alpha)})=(1-\alpha)+\alpha\lambda_{i}(H).
\]

In particular, for doubly stochastic matrices we have $\lambda_{1}=1$
and $\mid\lambda_{i}\mid\leq1$ for $i\geq2$. Defining the spectral
gap as $\gamma(H)=1-{\rm max}_{i\geq2}\lambda_{i}(H)$, it follows
for symmetric matrices that

\[
\gamma(H^{(\alpha)})=\alpha\gamma(H).
\]
In general, for absolute spectral gap $\gamma_{*}(H)=1-{\rm max}_{i\geq2}\left|\lambda_{i}(H)\right|$,
\[
\gamma_{*}(H^{(\alpha)})=1-{\rm max}_{i\geq2}\left|(1-\alpha)+\alpha\lambda_{i}(H)\right|.
\]
As $\alpha\to0$, $\gamma_{*}(H^{(\alpha)})\to0$. Thus, choosing
relatively small $\alpha$ reduces the spectral gap.

\section{Minorization}\label{sec:Minorization}

Let $H\in\mathbb{R}^{\text{n\texttimes n}}$ be a doubly stochastic
matrix, and define its $\epsilon$-minorized version \citep{meyn2009markov},
\citep{langville2004deeper} by

\[
H_{\epsilon}=(1-\epsilon)H+\epsilon J,\quad J=\frac{1}{n}\boldsymbol{1}\boldsymbol{1}\top,\quad\epsilon\in(0,1).
\]
The matrix $J$ is rank-one, doubly stochastic, and has spectrum ${1,0,\ldots,0}$.
Since both $H$ and $J$ preserve the vector $\boldsymbol{1}$, the
Perron eigenvalue 1 is also an eigenvalue of $H_{\epsilon}$, and
for all nontrivial eigenvalues $\lambda_{i}(H_{\epsilon})$, ($i\geq2$)
we have $\lambda_{i}(H_{\epsilon})=(1-\epsilon)\lambda_{i}(H)$
(see, for example, \citep{langville2004deeper}). Thus the absolute
spectral gap 
\[
\gamma(H_{\epsilon})=1-(1-\epsilon){\rm max}_{i\geq2}\left|\lambda_{i}(H)\right|=\gamma(H)+\epsilon{\rm max}_{i\geq2}\left|\lambda_{i}(H)\right|,
\]
so minorization increases the spectral gap (unless it was already
at maximum).

The effect of minorization is to introduce uniform mixing across all
coordinates. Intuitively, this suppresses deviations from the invariant
direction 1, since any component orthogonal to 1 is damped by the
averaging action of J. It also enforces ergodicity through positivity
of all elements.

\section{Implementations.}\label{sec:Implementations.}

\subsection{TBP implementations}

While mathematically TBP is optimal, the nonlinear structure of it
makes it stubborn to training. We will benchmark the performance of
several modifications of the pure algorithm.

\subsubsection{Scaled TBP (sTBP)}

The presence of $\Delta_{ij}$ factor in $x_{ij}=L_{ij}+\Delta_{ij}\,\sigma(t_{ij})$
makes the gradient w.r.t. $t_{ij}$ be also scaled by $\Delta_{ij}$, 

\[
\frac{\partial x_{ij}}{\partial t_{ij}}=\Delta_{ij}\,\sigma(t_{ij})(1-\sigma(t_{ij})).
\]

If $\Delta_{ij}$ is small, it reduces the sensitivity w.r.t the changes
of $t_{ij}$.

We rectify it by scaling the input of the sigmoid function in all
TBP implementations by an appropriate factor, 

\[
x_{ij}=L_{ij}+\Delta_{ij}\,\sigma\left(\frac{\beta\cdot t_{ij}}{\Delta_{ij}+\epsilon}\right)
\]
with $\beta=4$ and $\epsilon$ being small. Then the gradient

\[
\frac{\partial x_{ij}}{\partial t_{ij}}=\frac{\beta\,\Delta_{ij}}{\Delta_{ij}+\varepsilon}\,\sigma\!\left(\frac{\beta t_{ij}}{\Delta_{ij}+\varepsilon}\right)\left(1-\sigma\!\left(\frac{\beta t_{ij}}{\Delta_{ij}+\varepsilon}\right)\right)
\]
is not getting vanished, since $\frac{\Delta_{ij}}{\Delta_{ij}+\varepsilon}\approx1$.

\subsubsection{Margined TBP (mTBP)}

If $x_{ij}$ is close to the upper bound $U_{ij}$ - it automatically
drives all remaining elements in that row to 0, effectively fixing
them. To prevent this saturation, we use parameter $\rho\in\left(0,\frac{1}{2}\right)$
to constraint $x_{ij}\in[L_{ij}+\rho\Delta_{ij},~U_{ij}-\rho\Delta_{ij}]$,
i.e. $\frac{x_{ij}-L_{ij}}{\Delta_{ij}}\in(\rho,1-\rho)$, meaning
that the formula for $x_{ij}$ changes to

\[
x_{ij}=L_{ij}+\Delta_{ij}\,\left(\rho+(1-2\rho)\sigma\left(\frac{\beta\cdot t_{ij}}{\Delta_{ij}+\epsilon}\right)\right).
\]
This prevents $x_{ij}$ from sticking to the boundary too early.

\subsubsection{Averaged TBP (aTBP)}

TBP is biased toward the north-west corner of the matrix, which can
consume most of the budgets because of the order of parametrization.

A fix for this is to average it on a set of other orders, given by
permutations $\{\pi_{1},\dots,\pi_{K}\}$ of the numbers $\{1,2,\dots,n\}$.
Namely, for $\pi_{k}\in S_{n}$, let $P_{\pi_{k}}$ be the matrix
of that permutation on the basis of $\mathbb{R}^{n}$, and $z_{l}^{(k)}=\hat{x}'_{l}W_{l,k}^{\mathrm{res}}+b_{l,k}^{\mathrm{res}}$.
Then 
\[
\begin{array}{c}
H_{l}^{(k)}=\mathrm{TBP}(z_{l}^{(k)}),\\
\alpha_{l,k}={\rm softmax}\left(w_{l,k}^{a}\right)\\
H_{l}^{res}=\sum_{k=1}^{K}\alpha_{l,k}\,P_{\pi_{k}}^{\top}H_{l}^{(k)}P_{\pi_{k}}.
\end{array}
\]

It solves the order bias at the cost of increasing $K$-times the
number of parameters involved into computing $H_{l}^{res}$.

\subsubsection{Lazyfication-Minorisation (lmTBP)}

Because the mixing matrices in our architecture act as stochastic
transport operators, their spectral properties directly influence
training dynamics. In particular, the spectral gap - the separation
between the leading eigenvalue 1 and the subdominant eigenvalues -
controls the rate at which information is mixed across residual streams.

A small spectral gap implies that the mixing operator admits nontrivial
directions that are nearly invariant, i.e., eigenvectors associated
with eigenvalues close to one. As a result, information aligned with
these directions decays slowly across layers, leading to weak mixing
and metastable behavior. In analogy with Markov chains, this corresponds
to the presence of nearly disconnected components or clusters in the
transport dynamics.

Conversely, an excessively large gap induces rapid convergence toward
a low-dimensional invariant subspace, resulting in loss of expressivity
(analogous to oversmoothing in graph neural networks).

To regulate this trade-off, we use \emph{lazyfication} \citep{fried2021alphalazyversionmarkovchains,10.1561/0400000003,hermon2016maximal}
and \emph{minorisation} \citep{meyn2009markov,langville2004deeper}
of the mixing matrix (see appendices \ref{sec:Lazyfication}, \ref{sec:Minorization}).
Specifically, given a doubly stochastic matrix $\tilde{H}$, we replace
it with

\[
\begin{array}{c}
\widetilde{H_{l}}={\displaystyle \mathrm{TBP}(\hat{x}'_{l}W_{l}^{\mathrm{res}}+b_{l}^{\mathrm{res}})}\\
H_{l}^{\mathrm{res}}=(1-\lambda_{l}-\mu_{l})\widetilde{H}_{l}+\lambda_{l}I+\mu_{l}J,
\end{array}
\]
where $I$ is the identity and $J=\frac{1}{n}\mathbf{1}\mathbf{1}^{\top}$
is the uniform mixing operator. The lazyfication term $\lambda I$
increases aperiodicity and shrinks the magnitude of subdominant eigenvalues,
thereby opening a spectral gap and improving stability. The minorisation
term $\mu J$ enforces ergodicity by introducing a uniform mixing
component, ensuring a strictly positive spectral gap and preventing
degeneracies. Together, these modifications yield a controlled contraction
rate, stabilizing training while preserving sufficient diversity across
residual streams.

\subsubsection{Linearization (LTBP)}

We also try dropping $\sigma$ completely, leading to the linear formula
for the elements of $H_{l}^{res}$, 
\[
x_{ij}=L_{ij}+\Delta_{ij}\cdot t_{ij},
\]
with $t_{ij}$ being clipped to stay in the interval $\left[0;1\right]$.
This leads to simpler gradient flow with less saturation problems. 

Averaging and lazyfication-minorisation can also be applied to the
linear variant, leading to aLTBP, lmLTBP, and lmaLTBP (for both sequentially).

In LTBP and its variants $H_{l}^{res}$ does not depend on $X$, instead
being a fixed matrix per every layer of the model.

\subsubsection{Post-minorization}

All implementations also include post-minorization (see appendix \ref{sec:Minorization}),
to ensure that the resulting matrix is ergodic, with initial minorization
constant $\delta_{l}=\sigma(-8)\approx3.35\times10^{-4}$.

\subsection{RTBP implementation}

On the level of algorithms all modifications of TBP (s, m, a, lm)
can be applied to RTBP, giving sRTBP, mRTBP, etc.

Current implementation of RTBP variants for odd dimentions $n=2k+1$ of the matrix on the first step chips off last raw/column and passes to even-dimentional polytope. Since we only tested n=4 case, it did not affect the results.

Modern GPUs still have very little support for recursion. Therefore,
while the recursive formulation of the transportation Birkhoff parameterization
provides a conceptually clear description of the construction, our
implementation had to departs from a literal recursion in order to
better use the advantages of the modern GPU architectures. We restructure
the recursion into a vectorized, batched computation, where multiple
subproblems arising from the recursive decomposition are evaluated
simultaneously. This eliminates recursion overhead, improves memory
locality, and maximizes parallel utilization of GPU cores.

\subsubsection{Optimized RTBP (oRTBP)}

RTBP has a highly nonlinear structure with strong coupling between
parameters, which makes it difficult to train with standard optimization
methods. In particular, AdamW \citep{loshchilov2018decoupled}
weight decay does not interact well with the transport parameterization:
applying weight decay to the residual transport logits mainly pushes
them toward zero, which corresponds to midpoint choices in the local
transport intervals, rather than meaningfully shrinking the resulting
doubly stochastic matrix. As a result, the optimization dynamics of
the residual mixing matrix $H^{res}$ can become misaligned with the
intended geometry.

To address this, we introduce an optimized version of RTBP (oRTBP)
that makes the optimizer aware of this structure. The key idea is
to treat the residual transport parameters differently from the rest
of the model. Specifically, the parameters that define $H^{res}$
- including the transport logits, residual scale, and optional minorization
parameter $\delta$ - are placed in separate optimizer groups with
no weight decay. This allows us to assign them their own learning
rates, Adam hyperparameters, and gradient clipping thresholds. In
practice, we use distinct learning-rate multipliers for the transport
chart, the residual scale, and the $\delta$ parameter, and apply
separate gradient clipping to these groups.

By contrast, the projection matrices $H^{pre}$ and $H^{post}$ are
standard sigmoid-based mappings without geometric constraints, so
they behave like ordinary neural network layers and are trained with
default optimizer settings.

Overall, oRTBP does not change the underlying RTBP parameterization,
but rather provides an optimizer-aware implementation of the same
exact scaled recursive transport chart.

\section{Experiments.}\label{sec:Experiments.}

All experiments are run with distributed data parallelism on a single
node with 8 GPUs using \emph{torchrun}, a context length of 1024 tokens,
and bfloat16 training. We use a per-GPU micro-batch size of 16 and gradient\_accumulation\_steps=8;
this corresponds to an effective batch size of 131,072 tokens per
optimizer step. Optimization is performed with AdamW using beta1=0.9, beta2=0.95,
and weight decay 0.1. We apply cosine learning-rate decay with a 200-step
linear warmup, and set the minimum learning rate to one tenth of the
peak learning rate. Global gradient clipping is set to 1.0. Models
are evaluated every 500 steps over 200 evaluation batches.

We consider two main model scales. The small configuration uses 6
Transformer layers, 8 attention heads, and embedding dimension 512,
with peak learning rate $1e-3$. The medium configuration uses 12
layers, 12 heads, and embedding dimension 768, with peak learning
rate $6e-4$. Both settings are trained for 10,000 optimization steps,
corresponding to approximately 1.31B training tokens under the default
OpenWebText setup. 

Following \citep{mhc-lite}, in each training session losses are
computed as a moving average over the last 200 iterations.

In \emph{margined} TBP (mTBP, msTBP, amsTBP, etc.) variants $\rho=1e-4$.
For all \emph{averaged} TBP variants we only use two permutations,
identity and reverse, which gives direct and reverse order.

For oRTBP-mHC, we further use dedicated optimizer groups for the residual-chart,
residual-scale, and delta parameters, with details given in experiment
descriptions.

\subsection{Initialization.}

We initialize the hyper-connection parameters so that, at initialization,
each block starts from a simple near-residual regime. In the manifold-constrained
variants, $W_{l}^{res,k}$ , $W_{l}^{pre}$ and $W_{l}^{post}$ are
initialized to zero. The pre- and post-mapping biases $b_{l}^{pre}$
and $b_{l}^{post}$ are set to $-1$ in all entries except for one
designated stream, which is set to $1$, and their corresponding scales
$\alpha_{l}^{pre}$ and $\alpha_{l}^{post}$ are initialized to $0.01$.

For mHC, the residual-mixing bias matrix is initialized to $-8$ in
all entries except on the diagonal, which is set to $0$, i.e. $b_{l}^{\mathrm{res}}=-8\,\mathbf{1}\mathbf{1}^{\top}+8I.$Thus,
before learning begins, the Sinkhorn-Knopp projection is applied to
a matrix strongly biased toward the identity. 

For mHC-lite, the residual logits over permutation matrices are initialized
as $b_{l}^{\mathrm{res}}(1)=0,$ $b_{l}^{\mathrm{res}}(k)=-8,$ $k\neq1,$
where $P_{1}=I$ denotes the identity permutation. Hence the initial
convex combination is concentrated on $P_{1}$. 

For KromHC, the same initialization is applied factorwise: $b_{l}^{\mathrm{res},k}(1)=0,$
$b_{l}^{\mathrm{res},k}(m)=-8,$$m\neq1,$ so that each factor $U_{l}^{k}$
is initialized close to $I$, and therefore $H_{l}^{\mathrm{res}}=\bigotimes_{k=1}^{K}U_{l}^{k}\approx I.$

The TBP-chart variants adopt a different initialization. In these
models, the residual chart logits are initialized to zero, $b_{l}^{\mathrm{res}}=0,$and
the dynamic residual projections are also initialized to zero. In
LTBP sigmoid function is not used in the chart, so the parameters
are initialized to $0.5$. Thus the initial residual matrix is given
by the default chart construction, starting at the midpoint of each
feasible interval, rather than by an identity-biased initialization. 

When post-minorization is used, its logit is initialized as $\delta_{l}^{\mathrm{logit}}=-8,$so
that $\delta_{l}=\sigma(-8)\approx3.35\times10^{-4},$and the corresponding
uniform-mixing component is negligible at initialization. 

In all averaged variants (aTBP, asTBP, etc.), the chart-weight logits
are also initialized to zero, $w_{l,1}=\cdots=w_{l,K}=0,$yielding
$\alpha_{l,1}=\cdots=\alpha_{l,K}=\frac{1}{K}$ after the softmax,
i.e. an initially uniform weighting across the predefined charts.

\subsection{Experiment 1.}

seed: 1337, 2337, 3337, 4337, 5337

oRTBP optimizer group parameters: 

LR multipliers: residual chart: 1/6; residual scale: 0.187; delta
(minorization): 0.05

Custom gradient clipping: chart/scale: 0.3; delta: 0.05

GPU type: NVIDIA GeForce RTX 3090 24G

\begin{table}[h]
\centering
\caption{Experiment 1 (small scale). Mean, min, and max across 5 seeds.}
\label{tab:app-exp1}
\resizebox{\linewidth}{!}{%
\begin{tabular}{lccccccccc}
\hline
& \multicolumn{6}{c}{\textbf{Train}} & \multicolumn{3}{c}{\textbf{Validation}} \\
\cline{2-7} \cline{8-10}
\textbf{Model} 
& \multicolumn{3}{c}{loss} 
& \multicolumn{3}{c}{bpb} 
& \multicolumn{3}{c}{loss} \\
\cline{2-4} \cline{5-7} \cline{8-10}
& mean & min & max & mean & min & max & mean & min & max \\
\hline
mHC       
& 3.45259 & 3.40749 & 3.65089 
& 1.13591 & 1.13228 & 1.14350 
& 3.48348 & 3.45846 & 3.49576 \\

mHC-lite  
& 3.45531 & 3.39372 & 3.68160 
& 1.13696 & 1.13549 & 1.13917 
& 3.48742 & 3.46715 & 3.51147 \\

KromHC    
& 3.45584 & 3.39341 & 3.67852 
& 1.13694 & 1.13466 & 1.13868 
& 3.48722 & 3.46254 & 3.50939 \\

oRTBP-mHC 
& 3.45481 & 3.39034 & 3.67315 
& 1.13660 & 1.13484 & 1.13791 
& 3.48589 & 3.45970 & 3.50648 \\
\hline
\end{tabular}}
\end{table}

\subsection{Experiment 2.}

seed: 1337

oRTBP optimizer group parameters: 

LR multipliers: residual chart: 1/6; residual scale: 0.05; delta (minorization):
0.05

Custom gradient clipping: chart/scale: 0.3; delta: 0.05

GPU type: NVIDIA GeForce RTX 5090 32G

\begin{table}[h]
\centering
\caption{Experiment 2 (small scale). Validation and training performance with ranks.}
\label{tab:app-exp2}
\begin{tabular}{lcccccccc}
\hline
& \multicolumn{4}{c}{\textbf{Train}} & \multicolumn{4}{c}{\textbf{Validation}} \\
\cline{2-5} \cline{6-9}
\textbf{Model} & loss & rank & bpb & rank & loss & rank & bpb & rank \\
\hline
mHC         & 3.45308 & 2 & 1.13510 & 3 & 3.49746 & 3 & 1.13848 & 3 \\
mHC-lite    & 3.46016 & 9 & 1.13552 & 4 & 3.50663 & 9 & 1.14146 & 9 \\
KromHC      & 3.44823 & 1 & 1.13387 & 1 & 3.49513 & 1 & 1.13772 & 1 \\
oRTBP-mHC   & 3.45320 & 3 & 1.13472 & 2 & 3.49743 & 2 & 1.13846 & 2 \\
LTBP-mHC    & 3.45352 & 4 & 1.13554 & 5 & 3.49952 & 4 & 1.13914 & 4 \\
lmaLTBP-mHC & 3.45491 & 5 & 1.13601 & 6 & 3.50218 & 5 & 1.14001 & 5 \\
amsTBP-mHC  & 3.45783 & 6 & 1.13645 & 7 & 3.50336 & 6 & 1.14040 & 6 \\
msTBP-mHC   & 3.45893 & 7 & 1.13713 & 8 & 3.50599 & 7 & 1.14112 & 7 \\
asTBP-mHC   & 3.45913 & 8 & 1.13751 & 9 & 3.50599 & 8 & 1.14125 & 8 \\
\hline
\end{tabular}
\end{table}

\subsection{Experiment 3.}

seed: 1337

GPU type NVIDIA RTX PRO 6000 Blackwell Server Edition 96G

\begin{table}[h]
\centering
\caption{Experiment 3 (medium scale). Validation and training performance with ranks.}
\label{tab:app-exp3}
\begin{tabular}{lcccccccc}
\hline
& \multicolumn{4}{c}{\textbf{Train}} & \multicolumn{4}{c}{\textbf{Validation}} \\
\cline{2-5} \cline{6-9}
\textbf{Model} & loss & rank & bpb & rank & loss & rank & bpb & rank \\
\hline
mHC        & 3.24764 & 3 & 1.06566 & 3 & 3.27135 & 3 & 1.06487 & 3 \\
mHC-lite   & 3.24239 & 1 & 1.06397 & 1 & 3.26555 & 1 & 1.06299 & 1 \\
KromHC     & 3.25192 & 6 & 1.06716 & 6 & 3.27594 & 6 & 1.06637 & 6 \\
RTBP-mHC   & 3.25406 & 8 & 1.06760 & 7 & 3.27767 & 7 & 1.06667 & 7 \\
lmaLTBP-mHC& 3.24977 & 4 & 1.06623 & 4 & 3.27336 & 4 & 1.06553 & 4 \\
sRTBP-mHC  & 3.25422 & 9 & 1.06783 & 8 & 3.27778 & 9 & 1.06697 & 9 \\
msTBP-mHC  & 3.24537 & 2 & 1.06508 & 2 & 3.26911 & 2 & 1.06414 & 2 \\
asTBP-mHC  & 3.25782 & 10 & 1.06909 & 10 & 3.28178 & 10 & 1.06827 & 10 \\
aTBP-mHC   & 3.25061 & 5 & 1.06696 & 5 & 3.27438 & 5 & 1.06586 & 5 \\
TBP-mHC    & 3.25367 & 7 & 1.06795 & 9 & 3.27767 & 8 & 1.06693 & 8 \\
\hline
\end{tabular}
\end{table}

\subsection{Experiment 4.}

seed: 1337

oRTBP optimizer group parameters: 

LR multipliers: residual chart: 0.42; residual scale: 0.62; delta
(minorization): 0.25

Custom gradient clipping: chart/scale: 0.6; delta: 0.2

GPU type: NVIDIA RTX PRO 6000 Blackwell Server Edition 96G

\begin{table}[h]
\centering
\caption{Experiment 4 (medium scale). Final performance.}
\label{tab:app-exp4}
\begin{tabular}{lcccc}
\hline
\textbf{Model} & \textbf{Train loss} & \textbf{Train bpb} & \textbf{Val loss} & \textbf{Val bpb} \\
\hline
mHC        & 3.24528 & 1.06476 & 3.26951 & 1.06427 \\
mHC-lite   & 3.23809 & 1.06263 & 3.26105 & 1.06152 \\
oRTBP-mHC  & 3.25166 & 1.06710 & 3.27600 & 1.06639 \\
KromHC     & 3.27092 & 1.07423 & 3.29645 & 1.07304 \\
\hline
\end{tabular}
\end{table}

\subsection{Speed}

Contrary to the expectation, in small and medium configurations the
speed of TBP and RTBP variants was competitive. In experiments 1 and
4 oRTBP was on par with mHC (less than 2\% difference). In experiment
3 msTBP has shown similar performance. As expected, the averaged variants
(aTBP, lmaLTBP) were the slowest.

\begin{table}[h]
\centering
\caption{Throughput comparison (tokens/sec) across models and experiments.}
\begin{tabular}{lccc}
\toprule
\textbf{Model} & \textbf{Exp. 1} & \textbf{Exp. 3} & \textbf{Exp. 4} \\
\midrule
KromHC   & 242959 & 357970 & 379204 \\
mHC-lite & 240720 & 358385 & 379918 \\
oRTBP    & 215843 & ---    & 329379 \\
msTBP    & ---    & 320125 & ---    \\
mHC      & 215374 & 317481 & 335103 \\
\bottomrule
\end{tabular}
\end{table}

\subsection{Gradient-unstable KromHC runs}\label{subsec:Gradient-unstable-KromHC-runs}

5746: NVIDIA RTX PRO 6000 Blackwell Server Edition 96G

7234: NVIDIA GeForce RTX 3090

\noindent\begin{minipage}[t]{1\columnwidth}%
\begin{minipage}[t]{0.5\columnwidth}%
\includegraphics[scale=0.27]{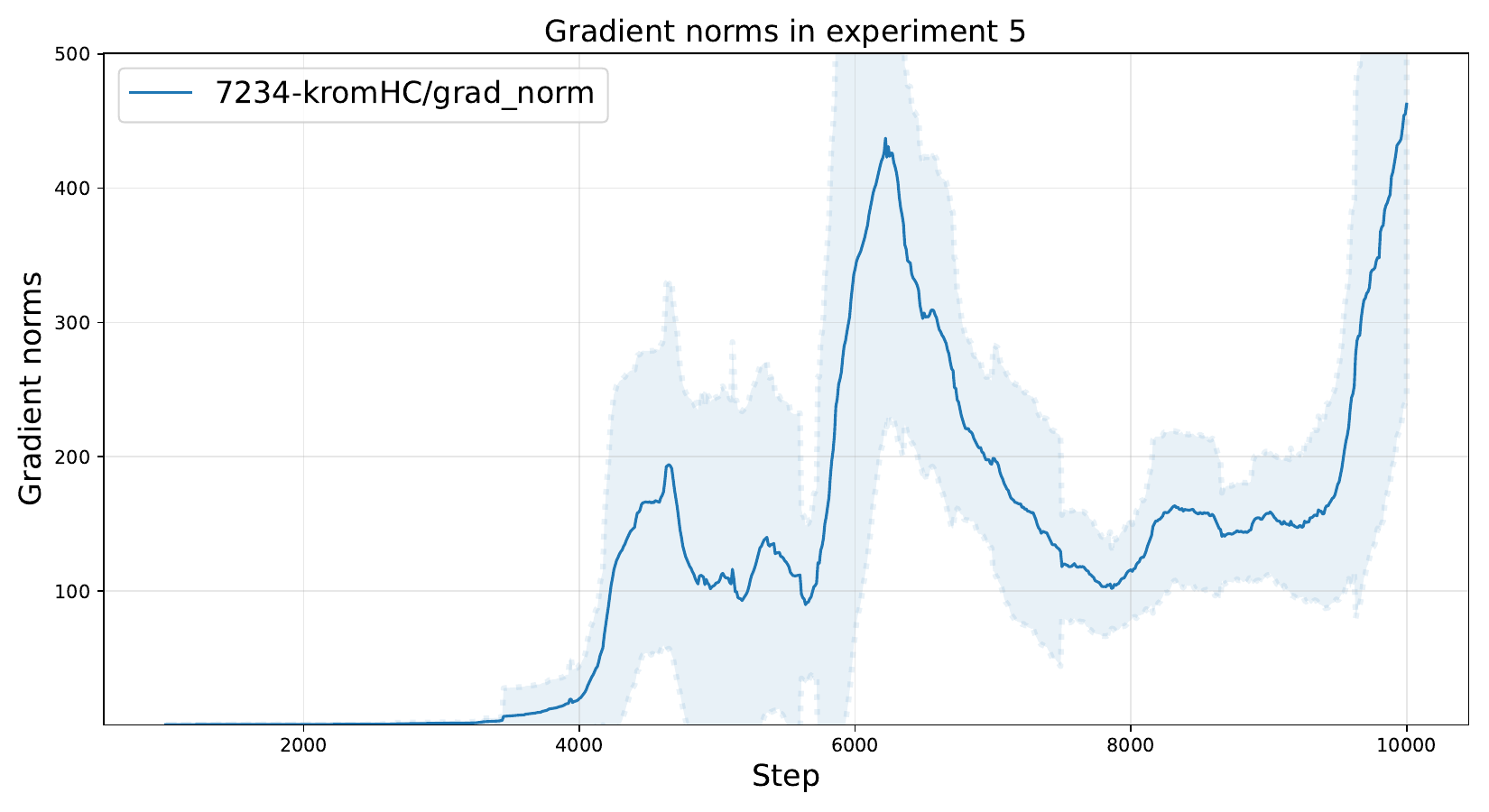}%
\end{minipage}%
\begin{minipage}[t]{0.5\columnwidth}%
\includegraphics[scale=0.27]{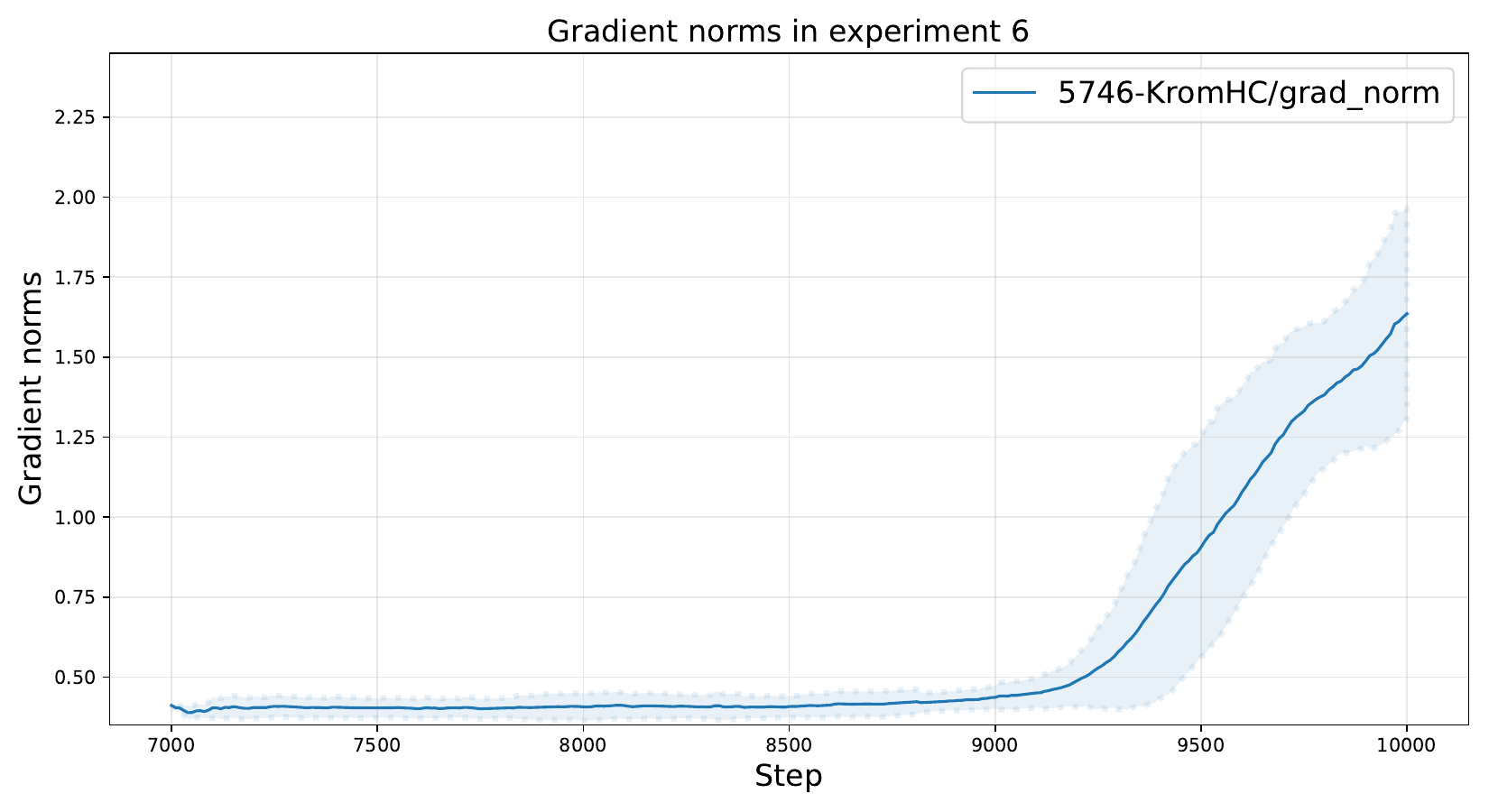}%
\end{minipage}%
\end{minipage}

\section{Algorithms}\label{sec:Algorithms}

\begin{algorithm}[H]
\caption{Transportation Birkhoff Polytope (TBP)}
\begin{algorithmic}[1]
\Require $n \ge 2$, free parameters $t_{ij}$ for $1 \le i,j \le n-1$
\State Initialize $r_i \gets 1$, $c_j \gets 1$ for all $1 \le i,j \le n$
\For{$i=1$ to $n-1$}
    \For{$j=1$ to $n-1$}
        \State $R^{\text{tail}}_j \gets \sum_{k=j+1}^{n} c_k$
        \State $C^{\text{tail}}_i \gets \sum_{k=i+1}^{n} r_k$
        \State $L_{ij} \gets \max\bigl(0,\; r_i - R^{\text{tail}}_j,\; c_j - C^{\text{tail}}_i \bigr)$
        \State $U_{ij} \gets \min(r_i,\; c_j)$
        \State $x_{ij} \gets L_{ij} + (U_{ij}-L_{ij})\,\sigma(t_{ij})$
        \State $r_i \gets r_i - x_{ij}$
        \State $c_j \gets c_j - x_{ij}$
    \EndFor
    \State $x_{i,n} \gets r_i$
    \State $c_n \gets c_n - x_{i,n}$
\EndFor
\For{$j=1$ to $n-1$}
    \State $x_{n,j} \gets c_j$
\EndFor
\State $x_{n,n} \gets r_n$
\State \Return $X=(x_{ij})$
\end{algorithmic}
\end{algorithm}

\medskip{}

\pagebreak{}

\begin{algorithm}[t]
\caption{Recursive Transportation Birkhoff Polytope (RTBP)}
\label{alg:rtbp}
\begin{algorithmic}[1]
\Require Row margins $r\in\mathbb{R}^n_{>0}$, column margins $c\in\mathbb{R}^m_{>0}$ with $\sum_i r_i=\sum_j c_j$, free parameters $t$
\Ensure Matrix $X\in\mathcal{T}(r,c)$

\Function{RTBP}{$r,c,t$}
    \State $n\gets |r|$, $m\gets |c|$

    \If{$n=1$}
        \State \Return $X=(c_1,\dots,c_m)$
    \EndIf

    \If{$m=1$}
        \State \Return $X=(r_1,\dots,r_n)^{\top}$
    \EndIf

    \If{$n$ is even}
        \State $k\gets n/2$
    \Else
        \State $k\gets (n+1)/2$
    \EndIf

    \If{$m$ is even}
        \State $\ell\gets m/2$
    \Else
        \State $\ell\gets (m+1)/2$
    \EndIf

    \State $I_1\gets\{1,\dots,k\}$, $I_2\gets\{k+1,\dots,n\}$
    \State $J_1\gets\{1,\dots,\ell\}$, $J_2\gets\{\ell+1,\dots,m\}$

    \State $R_1\gets \sum_{i\in I_1}r_i$, \quad $R_2\gets \sum_{i\in I_2}r_i$
    \State $C_1\gets \sum_{j\in J_1}c_j$, \quad $C_2\gets \sum_{j\in J_2}c_j$

    \State $L\gets \max(0,R_1-C_2,C_1-R_2)$
    \State $U\gets \min(R_1,C_1)$

    \State Choose next free parameter $t_0$ from $t$
    \State $M_{11}\gets L+(U-L)\sigma(t_0)$
    \State $M_{12}\gets R_1-M_{11}$
    \State $M_{21}\gets C_1-M_{11}$
    \State $M_{22}\gets R_2-M_{21}$

    \State Choose vectors $r'_{I_1},r''_{I_1},r'_{I_2},r''_{I_2}$ such that
    \[
        r'_{I_1}+r''_{I_1}=r_{I_1},
        \qquad
        r'_{I_2}+r''_{I_2}=r_{I_2},
    \]
    \[
        \sum_{i\in I_1}r'_i=M_{11},
        \quad
        \sum_{i\in I_1}r''_i=M_{12},
        \quad
        \sum_{i\in I_2}r'_i=M_{21},
        \quad
        \sum_{i\in I_2}r''_i=M_{22}.
    \]

    \State Choose vectors $c'_{J_1},c''_{J_1},c'_{J_2},c''_{J_2}$ such that
    \[
        c'_{J_1}+c''_{J_1}=c_{J_1},
        \qquad
        c'_{J_2}+c''_{J_2}=c_{J_2},
    \]
    \[
        \sum_{j\in J_1}c'_j=M_{11},
        \quad
        \sum_{j\in J_2}c'_j=M_{12},
        \quad
        \sum_{j\in J_1}c''_j=M_{21},
        \quad
        \sum_{j\in J_2}c''_j=M_{22}.
    \]

    \State Split remaining parameters $t$ into $t_{11},t_{12},t_{21},t_{22}$

    \State $X_{11}\gets \Call{RTBP}{r'_{I_1},c'_{J_1},t_{11}}$
    \State $X_{12}\gets \Call{RTBP}{r''_{I_1},c'_{J_2},t_{12}}$
    \State $X_{21}\gets \Call{RTBP}{r'_{I_2},c''_{J_1},t_{21}}$
    \State $X_{22}\gets \Call{RTBP}{r''_{I_2},c''_{J_2},t_{22}}$

    \State \Return
    \[
        X=
        \begin{pmatrix}
            X_{11} & X_{12}\\
            X_{21} & X_{22}
        \end{pmatrix}
    \]
\EndFunction
\end{algorithmic}
\end{algorithm}


\end{document}